\documentclass[11pt,a4paper]{article}
\usepackage[hyperref]{acl2021}
\usepackage{times}
\usepackage{soul}
\usepackage{latexsym}
\usepackage{amsmath}
\usepackage{amssymb}

\usepackage{physics}
\usepackage{bbm, dsfont}
 
\usepackage{multirow}
\usepackage{amsmath}
\usepackage{algorithm}
\usepackage{amsfonts}
\usepackage{algpseudocode}
\usepackage{amsthm}
\usepackage{textgreek}
\usepackage{relsize}
\usepackage{booktabs}
\usepackage{multicol} % For two-column layout
\usepackage[titletoc]{appendix} % For appendix formatting

\newtheorem{thm}{Theorem}
\usepackage[noabbrev,capitalize,nameinlink]{cleveref}
\usepackage{tikz}
\newcommand*\circled[1]{\tikz[baseline=(char.base)]{
            \node[shape=circle,draw,inner sep=1.4pt] (char) {#1};}}

\usepackage{graphicx}

\usepackage{microtype}

\usepackage{amsmath}
\usepackage{cleveref}

\crefname{ablation}{ablation}{ablations}
\crefname{equation}{Expression}{Expressions}
\crefname{alg}{algorithm}{algorithms}

\DeclareMathOperator{\M}{M}
\crefname{thm}{theorem}{theorems}

\title{Towards Probabilistically-Sound Beam Search with Masked Language Models}

\begin{document}

\aclfinalcopy
\maketitle

\begin{abstract}
Beam search with masked language models (MLMs) is challenging in part because joint probability distributions over sequences are not readily available, unlike for autoregressive models. However, estimating such distributions has important domain-specific applications such as ancient text restoration and protein engineering. Here we present probabilistically-sound methods for beam search with MLMs. First, we clarify the conditions under which it is theoretically sound to perform text infilling with MLMs using standard beam search. When these conditions fail, we provide a probabilistically-sound inference time modification with no additional computational complexity and demonstrate that it is superior to the aforementioned beam search in the expected conditions. We then present empirical results comparing several infilling approaches with MLMs across several domains. Notably, our method probes the inductive biases of MLMs and explores the surprising contextual sensitivity of mask tokens for text infilling. 

\end{abstract}

\section{Introduction}
Autoregressive language models (LMs) have demonstrated success in many tasks. Yet in specific contexts where bidirectionality is crucial, such as ancient text restoration and protein engineering, masked language models (MLMs) are most prevalent. Notably, MLMs still face a significant challenge in these settings: while MLMs learn conditional distributions of single tokens, applications to the aforementioned domains often require knowledge of the probability of multiple tokens jointly. 

The key challenge lies in computing the joint distribution $p(\mathbf{x})$ over sequences $\mathbf{x}$ given only the MLM-learned conditionals $p_i(x_i|\mathbf{x}_{-i})$, where $\mathbf{x}_{-i}$ denotes the context sequence $\mathbf{x}$ with the entry at index $i$ removed. The Hammersley-Clifford-Besag (HCB) theorem \citep{besag_spatial_1974} yields a direct algebraic construction of a joint distribution $p(\mathbf{x})$ which is valid in the case when the MLM-learned distributions are \textit{compatible}---that is, there exists a joint distribution which factors into the conditionals exactly \cite{hennigen_deriving_2023}.

\begin{figure}
    \centering
    \includegraphics[width=1.05\linewidth]{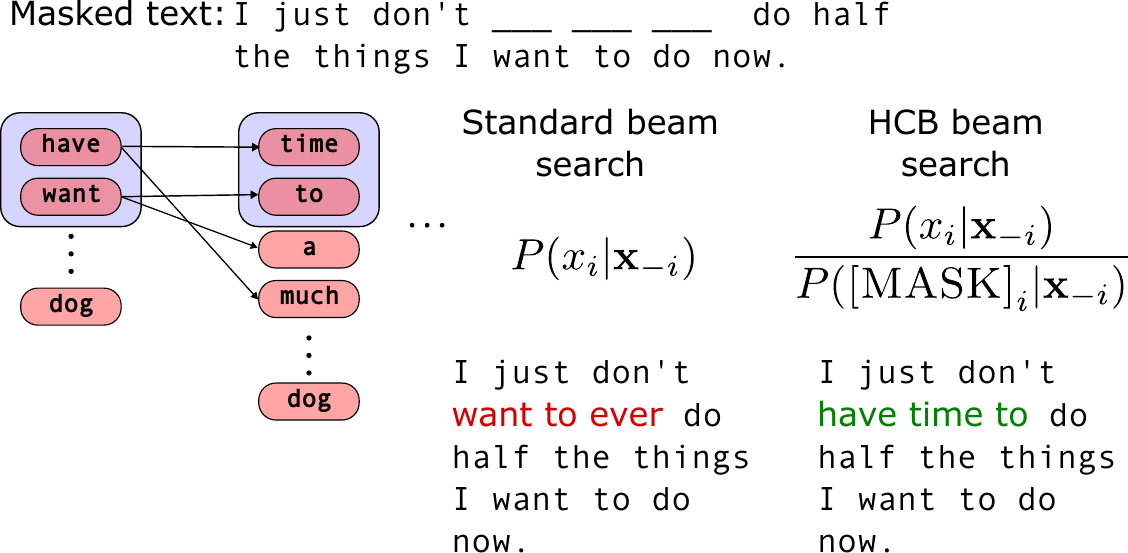}
    \caption{Overview of the proposed HCB beam search compared to standard beam search for text infilling.}
    \label{fig:schematic}
\end{figure}

In practical applications, a heuristic approximation based on the chain rule of autoregressive LMs is used for infilling with MLMs \citep{Ruffolo_Madani_2024, Shen_Quach_Barzilay_Jaakkola_2020, assael2022restoring, cowen-breen_logion_2023, wang_bert_2019}:
\begin{equation}\label{eqn:app}
    p(\mathbf{x}) \approx \prod_{i=1}^n p(x_i | \mathbf{x}_{:i}, \mathbf{[M]}_{i:})
\end{equation}
where the notation $\mathbf{[M]}_{i:}$ indicates that mask tokens are present from indices $i$ onwards.

In this work, we demonstrate that \autoref{eqn:app} is valid if and only if a \textit{conditional independence assumption} about the MLM-learned conditionals is satisfied and provide the conditions under which this assumption holds, assuming compatibility. We state these conditions in \Cref{thm:inf}:

\begin{thm}[Informal]\label{thm:inf}
Suppose that $p$ represents a model which achieves minimal training loss on the MLM objective. Then, on the training distribution, the learned conditionals are both \emph{compatible} and satisfy the \emph{conditional independence assumption}.
\end{thm}

We hypothesize that there may be two regimes as a consequence of \Cref{thm:inf}: a regime where training loss is small enough that the heuristic approach of \autoref{eqn:app} may be reasonable, and a regime where the conditional independence assumption may not hold, in which another approach is needed. In this work, we provide a modification of \autoref{eqn:app} based on the HCB theorem in the second regime, which relaxes the conditional independence assumption by including an adjustment term to correct for possible dependencies and requires no additional forward passes through the model. We present this modification, which we call \textbf{HCB beam search}, and the resulting theoretical and empirical contributions to zero-shot text infilling. 

\paragraph{Theoretical contributions.}
Previous work has evaluated text scoring rather than text infilling, considering proxies for joint distributions to evaluate the probabilities of pairs of tokens and computing the perplexities of these joint distributions \cite{hennigen_deriving_2023}. A key insight of our work is that the HCB construction is the only such proxy which decomposes in a way suitable for beam search. It is also particularly useful in low-resource settings, as it can be implemented with no additional computational requirements, allowing for the adaptation of generalist MLMs for the text infilling task. We present a theorem classifying the conditions under which the compatibility of learned distributions hold and showing that standard beam search with MLMs is justified precisely when the conditions of this theorem hold. 

% despite $\textsc{[mask]}$ never appearing as a label during training. 

\paragraph{Empirical contributions.}
We find empirically that our modification outperforms standard beam search in various contexts for several models, including BERT-base, BERT-large \citep{devlin_bert_2019}, DistillBERT \citep{sanh2019distilbert}, Ithaca \citep{assael2022restoring}, and Desformers \citep{devaul_desformers}. In exploring how HCB beam search fares across several models, we uncover model-specific performance results. To our knowledge, this work is also the first to investigate $p(\textsc{[mask]}| \cdot)$, demonstrating through \hyperlink{ablation1}{Ablation 1} and \hyperlink{ablation2}{Ablation 2} that it conveys a strongly context-specific signal beyond random noise which can improve beam search results. 

% Infilling is a domain-motivated task, and domain practitioners are typically restricted either by data or computational constraints, necessitating methods to perform infilling with MLMs. 

% Finally, previous work did not consider the choice of \textsc{[mask]} as a pivot; to our knowledge, our work is also the first to show that $p(\textsc{[mask]})$ conveys any meaningful information, as demonstrated by our ablations. In general, we find that the choice of pivot strongly affects the effectiveness of HCB beam search, and we provide comparisons across different choices of pivots; effects of pivot choice have not been previously described. 

\section{Text Infilling}
We experimentally evaluate \autoref{eqn:app} and our modification of it via \textit{text infilling}, the task of predicting a missing span of text given its context. This task is well-motivated in the domains of protein language modeling and ancient text restoration \cite{zhu2019text}.

Much existing work on text infilling focuses on developing custom architectures and training from scratch \cite{Sun_Lee_Batra_2017, Ippolito_Grangier_Callison-Burch_Eck_2019, Shen_Quach_Barzilay_Jaakkola_2020, Donahue_Lee_Liang_2020, He_Sun_Tang_Wang_Huang_Qiu_2023}. In some low-resource domains, however, pretrained MLMs may be available when limitations on available data or compute prohibit training new text infilling models. Manuscript restoration is a prototypical example: training data is often restricted by copyright \citep{Graziosi_Haubold_Cowen-Breen_Brooks_2023}, and full trainings can be computationally expensive, yet various pretrained MLMs for ancient languages are available \cite{Bamman_Burns_2020, cowen-breen_detecting_2023, riemenschneider2023exploring}. Another example is protein language models (PLMs): the computational footprint of fine-tuning PLMs becomes a barrier for many research groups \cite{hu_lora_2021, sledzieski_democratizing_2023}. Despite the challenges posed by data and compute, effective infilling remains an important task for ancient text restoration and protein engineering.

% Here, we examine the capabilities of MLMs to infill directly, primarily by infilling tokens sequentially through an adaptation of beam search to MLMs, although we compare additionally to other sampling schemes, such as nucleus sampling and sampling with temperature.

% While our method only applies to settings where the number of missing tokens is fixed, as opposed to the more general setting addressed by others (\citet{Shen_Quach_Barzilay_Jaakkola_2020}), we find that this is a realistic assumption in several important applications, such as ancient text restoration and protein engineering. In the case of damaged inscriptions, domain experts posit an estimated number of missing characters based on physical distance \cite{bruun2014oxford}. In protein engineering, a reasonable constraint is maintaining the overall length of the protein while only modifying amino acids in functional regions. In both areas, the fixed length assumption has been used in cutting-edge models \citep{assael2022restoring, wang2022scaffolding}. %https://www.science.org/doi/10.1126/science.abn2100%

\section{Background and Related Work}
\subsection{Beam search}
Beam search is a form of decoding which incrementally adds successive tokens to a set of candidate token sequences,  maintaining a fixed number of candidate sequences $\mathbf{x}=(x_1,\hdots,x_n)$ of highest joint probability $p(\mathbf{x})$. For autoregressive models, the chain rule,\footnote{In what follows, we generally omit right-context tokens from $p(\cdot|\cdot)$ for readability, but we note that modified forms of all of the following equations hold when $p$ is further conditioned on right-context.}
\begin{equation}\label{eqn:chain}
    \log p(\mathbf{x}) = \sum_{i=1}^n \log p(x_i|\mathbf{x}_{:i}),
\end{equation}
allows for the well-known autoregressive beam search algorithm (\autoref{sec:app_algorithms}, \cref{alg:auto}).

\subsection{Challenges with MLM beam search}\label{subsec:MLMbeam}
A major barrier to conducting beam search with MLMs is that MLMs are not language models \textit{a priori}, and thus it is less obvious that a joint distribution $p(\mathbf{x})$ exists and can be computed in terms of known quantities. If one is willing to tolerate compatibility and a \textbf{conditional independence assumption}\footnote{We investigate this assumption in \Cref{sec:discussion}. While \autoref{eqn:approx} may seem intuitively clear, it might be similarly intuitive that the conditionals learned by BERT are compatible, but this is far from true \cite{young2023inconsistencies}.} of the form\begin{equation}\label{eqn:approx}
    p(x_i | \mathbf{x}_{:i}, \mathbf{[M]}_{i:}) \approx  p(x_i | \mathbf{x}_{:i})
\end{equation}
then \autoref{eqn:chain} implies the following expression for the joint:\footnote{For a more detailed derivation, see \autoref{sec:proof}.}
\begin{equation}\label{eqn:approxbeam}
    \log p(\mathbf{x}) \approx \sum_{i=1}^n \log p(x_i | \mathbf{x}_{:i}, \mathbf{[M]}_{i:}) 
\end{equation}
\autoref{eqn:approxbeam} is the foundation for the implementation of beam search shown in \cref{alg:beam} with the standard scoring function.

\begin{figure*}
    \begin{minipage}{\textwidth}
\begin{algorithm}[H]
\small
\caption{Infilling beam search with a MLM. Given $\mathbf{x}$, a sequence of length $n$ with masked positions ${j, ..., k}$ and a beam size $B>0$, return a collection of generated sequences $S$ with masked positions filled in from vocabulary $V$. Uses a scoring function $f(\cdot)$ to evaluate a candidate beam extension.}\label{alg:beam}
\begin{algorithmic}
\State Initialize $S = \{(0,\emptyset)\}$
\For{$i\in\{j,\hdots,k\}$}
    \For{$(\ell,(x_j,\hdots,x_{i-1}))\in S$}
    \State Append to $S$: $(\ell + f(x), (x_j,\hdots,x_{i-1},x))$ for every $x\in V$.
    \EndFor
    \State $S \gets$ $\{$the $B$ sequences $(x_j,\hdots,x_i)$ of $S$ w/ highest $\ell\}$
\EndFor \\
\Return S
\end{algorithmic}
\end{algorithm}
    \end{minipage}

\begin{minipage}{\textwidth}
\vspace{2mm} 
Scoring functions $f$ for various beam search implementations:\footnotemark
        \begin{align*}
\text{Standard :\quad\quad} f(\cdot) &= \log p(\cdot|\mathbf{x}_{:i},\mathbf{[M]}_{i:k},\mathbf{x}_{k:})  \\
\text{HCB (ours) :\quad\quad} f(\cdot) &= \log p(\cdot|\mathbf{x}_{:i},\mathbf{[M]}_{i:k},\mathbf{x}_{k:}) - \log p(\mathbf{[M]}_i|\mathbf{x}_{:i},\mathbf{[M]}_{i:k},\mathbf{x}_{k:}) \\
\text{HCB with pivot $\mathbf{y}_{i:k}$ (ours) :\quad\quad} 
f(\cdot) &= \log p(\cdot|\mathbf{x}_{:i},\mathbf{y}_{i:k},\mathbf{x}_{k:}) - \log p(\mathbf{y}_i|\mathbf{x}_{:i},\mathbf{y}_{i:k},\mathbf{x}_{k:})
        \end{align*}
        \rule{\linewidth}{0.4pt}
    \end{minipage}
\end{figure*}\footnotetext{The notation $\mathbf{[M]}_{i:k}$ indicates that \textsc{[mask]} tokens occupy the indices from $i$ to $k$. Unless otherwise specified, the left-most mask index ($i$) is the one at which the probability $p(\cdot|\mathbf{[M]}_{i:k})$ will be evaluated.}

The approximation in \autoref{eqn:approx} is equivalent to the assumption that the distribution of $x_i$ conditioned on the given context $\mathbf{x}_{:i}$ is independent of the information that mask tokens occupy the indices from $i$ to $n$. It is unlikely that this assumption holds true in practice, as passing mask tokens to the model will alter the output distribution in general.

Therefore, to be probabilistically sound, this equation should include a term correcting for the potential dependency between $x_i$ and $\mathbf{[M]}_{i:}$. Including this term conveniently incurs almost no additional computational cost, to be described in \Cref{sec:methods}.

\subsection{Constructing joint distributions from conditionals}
Approximating the joint distributions of MLMs is an active area of research.
\citet{hennigen_deriving_2023} compare several joint distribution approximation schemes and find that when the MLM-learned conditional distributions are compatible---that is, there exists a joint distribution which factors into the conditionals exactly---the HCB theorem provides a direct algebraic construction of a joint distribution $p(\mathbf{x})$ from a set of conditionals $p(x_i|\mathbf{x}_{-i})$, up to a normalizing constant.

\begin{thm}[HCB]\label{thm:HCB}
Suppose that $p$ is a probability distribution with full support over the space $A^n$ of $n$-element sequences over an arbitrary alphabet $A$. Then for any two sequences $\mathbf{x},\mathbf{y}\in A^n$,
\[
\frac{p(\mathbf{x})}{p(\mathbf{y})} = \prod_{i=1}^n \frac{p(x_i | x_1,\hdots,x_{i-1},y_{i+1},\hdots,y_n)}{p(y_i | x_1,\hdots,x_{i-1},y_{i+1},\hdots,y_n)}
\]
\end{thm}

\noindent \Cref{thm:HCB} yields an immediate expansion for $\log p(\mathbf{x})$ which is an analog to \autoref{eqn:chain} for MLMs:
\begin{align}\label{eqn:pchain}
\begin{split}
    \log p(\mathbf{x}) \sim \sum_{i=1}^n &\log p(x_i | \mathbf{x}_{:i},\mathbf{y}_{i+1:})\\
     -& \log p(y_i | \mathbf{x}_{:i},\mathbf{y}_{i+1:})
\end{split}
\end{align}
where $\sim$ indicates equality up to addition of a constant $\mathbf{x}$ (equal to $\log p(\mathbf{y})$). Following \citet{hennigen_deriving_2023}, we refer to $\mathbf{y}$ as the \textbf{pivot}. When the conditional distributions are compatible, \autoref{eqn:pchain} should yield orderings of sequences $\mathbf{x}$ by their probabilities $p(\mathbf{x})$ in a manner which is consistent across choice of pivots.

In actuality, the conditional distributions learned by BERT do not appear to be compatible \cite{young2023inconsistencies, hennigen_deriving_2023}. When the MLM-learned conditional distributions are not compatible, the Arnold-Gokhale (AG) construction provides an algorithm for returning the joint distribution which \textit{most nearly} factors into the learned conditionals \cite{Arnold_Gokhale_1998}. \citet{hennigen_deriving_2023} find that the AG construction achieves the lowest perplexity when compared to a number of baselines; however, it is severely limited by the computational cost it incurs: memory requirements of $V^n$ for a vocabulary of size $V$ and $n$ missing tokens.

% \url{https://github.com/rcalef/hcb_infilling}

\section{Methods}\label{sec:methods}
Our primary observation is that the HCB theorem (\Cref{thm:HCB}) yields a straightforward correction to the standard beam search induced by \autoref{eqn:approxbeam} which incurs almost no additional computational cost.\footnote{Code for our experiments can be found at \url{https://github.com/rcalef/hcb_infilling}} Although the conditional distributions learned by BERT are not exactly compatible, they are empirically compatible enough to improve the accuracy of beam search in certain instances.\footnote{It is worth noting that standard beam search assumes compatibility to the same extent that HCB does.} Thus, these methods can be useful for infilling when the implementation of the AG construction is intractable.

Based on \Cref{thm:HCB} and \autoref{eqn:pchain}, we propose the following modification to \autoref{eqn:approxbeam}:

\vspace{-2em}
\begin{align}\label{eqn:mchain}
    \begin{split}
        \log p(\mathbf{x}_{j:k}|\mathbf{x}_{:j},\mathbf{x}_{k:}) \sim \sum_{i=j}^k &\log p(x_i | \mathbf{x}_{:i},\mathbf{[M]}_{i:k},\mathbf{x}_{k:}) \\
        -&\log p([\M]| \mathbf{x}_{:i},\mathbf{[M]}_{i:k},\mathbf{x}_{k:})
    \end{split}
\end{align}
\vspace{-1em}

where $\sim$ again indicates equality up to addition of a constant in the tokens-to-be-infilled $\mathbf{x}_{j:k}$, equal to $\log p([\mathbf{M}]_{i:k} | \mathbf{x}_{:j}, \mathbf{x}_{k:})$.
This correction term guarantees probabilistic soundness of the type ensured by \autoref{eqn:chain} and requires no additional forward passes to compute, as the tensor $p(\cdot | \mathbf{x}_{:i},\mathbf{[M]}_{i:})$ is computed with a single forward pass. Although the value of $p([\M]| \mathbf{x}_{:i},\mathbf{[M]}_{i:})$ does not have an immediately obvious intuitive meaning, \hyperlink{ablation1}{Ablation 1} and \hyperlink{ablation2}{Ablation 2} show that it depends strongly on the context $\mathbf{x}_{:i}$  and can be used to improve over baselines.

\subsection{Choosing a pivot}
For the purposes of sampling and infilling, the pivot $\mathbf{y}$ can be any sequence in the support of $p$. However, we find that some pivots lead to better performance than others (\autoref{fig:pivots}). During the computation of \autoref{eqn:mchain}, the pivot $\mathbf{y}$ is injected into the text as context. Therefore, we find it important that the pivot is reasonably in-distribution as context for the model, regardless of the position where it is injected.

The MLM training procedure makes one particular choice of pivot especially in-distribution for any context: the sequence of mask tokens $\mathbf{y}=([\M],\hdots,[\M])$. Throughout MLM training, mask locations are re-randomized, meaning that the MLM is likely to encounter sequences of masks as context in many positions across all examples in the train set. The probability $p([\M]|\mathbf{x}_{<i},\mathbf{y}_{>i})$ decreases drastically during training, as $[\M]$ never occurs as a label.\footnote{See Figure \ref{fig:mask_probs} for how $p([\M| \cdot])$ changes during BERT training. Note that while $[\M]$ is never seen as a label during training, it is still awarded positive probability on account of the softmax in the final layer of BERT.} Nonetheless, using this quantity improves infilling accuracy across different models in various domains. We perform ablation studies to verify that this probability captures genuine information about the context, rather than random noise.

% \begin{algorithm}
% \small
% \caption{HCB beam search. Given a beam size $B>0$, it returns a collection of generated sequences $S$ of length $n$.}\label{alg:modified}
% \begin{algorithmic}
% \State Initialize $S = \{(0,\emptyset)\}$
% \For{$i\in\{1,\hdots,n\}$}
% \For{$(\ell,(y_1,\hdots,y_{i-1}))\in S$}
% \State $f(\cdot) \gets \log P(\cdot | {y_1,\hdots,y_{i-1},\text{[M]}_{i+1},\hdots,\text{[M]}_n})$ with one forward pass.\footnotemark
% \State Append to $S$: $(\ell + f(y) - f([M]), (y_1,\hdots,y)) \forall y$.
% \EndFor
% \State $S \gets$ $\{$the $B$ sequences $(y_1,\hdots,y_i)$ of $S$ w/ highest $\ell\}$
% \EndFor \\
% \Return S
% \end{algorithmic}
% \end{algorithm}\footnotetext{This is a tensor of length $|V|$ which assigns a probability to every token.}

\begin{table*}[]
    \small
    \centering
    \begin{tabular}{|c||c|c|c|c||c|c|c|c|}
    \hline
    \textit{Number of missing tokens} & 2 & 3 & 4 & 5 & 2 & 3 & 4 & 5 \\
    \hline
    & \multicolumn{4}{c||}{Top-1 Accuracy (\%)} & \multicolumn{4}{c|}{Top-5 Accuracy (\%)} \\
    \hline
    HCB Left-to-Right & 28.84 & 10.46 & 3.23 & 0.92 & 38.88 & 15.01 & 4.75 & 1.36 \\
    HCB Best-to-Worst & \textbf{29.76} & \textbf{11.29} & \textbf{3.66} & \textbf{1.02} & \textbf{39.73} & \textbf{16.20} & \textbf{5.40} & \textbf{1.56} \\
    Standard Left-to-Right & 26.43 & 8.93 & 2.69 & 0.74 & 36.80 & 13.77 & 4.30 & 1.21 \\
    Standard Best-to-Worst & 28.27 & 10.21 & 3.18 & 0.86 & 38.18 & 15.08 & 4.88 & 1.36 \\
    \hline
    Pure sampling & 23.71 & 6.77 & 1.58 & 0.30 & 28.78 & 7.92 & 1.73 & 0.33 \\
    Sampling with T=0.25 & 26.10 & 8.41 & 2.46 & 0.68 & 34.23 & 11.08 & 3.18 & 0.84 \\
    Sampling with T=0.50 & 25.89 & 8.23 & 2.22 & 0.68 & 33.38 & 10.61 & 2.96 & 0.78 \\
    Sampling with T=0.75 & 25.09 & 7.63 & 1.99 & 0.46 & 31.54 & 9.42 & 2.43 & 0.53 \\
    Nucleus sampling, p=0.90 & 24.31 & 7.23 & 1.87 & 0.40 & 27.91 & 8.49 & 2.08 & 0.45 \\
    \hline
    & \multicolumn{4}{c||}{BLEU} & \multicolumn{4}{c|}{BERTScore F1} \\
    \hline
    HCB Left-to-Right & 28.74 & 10.33 & 3.25 & 0.93 & \textbf{99.47} & \textbf{99.02} & 98.66 & \textbf{98.36}  \\
    HCB Best-to-Worst & \textbf{29.55} & \textbf{11.13} & \textbf{3.63} & \textbf{1.03} & 99.46 & \textbf{99.02} & \textbf{98.72} & 98.32 \\
    Standard Left-to-Right & 26.42 & 8.94 & 2.68 & 0.74 & 99.39 & 98.96 & 98.63 & 98.29 \\
    Standard Best-to-Worst & 28.15 & 10.05 & 3.12 & 0.89 & 99.43 & 98.98 & 98.64 & 98.29 \\
    \hline
    \end{tabular}
    \caption{Top-$k$ \% accuracy, BLEU score, and BERTScore at infilling consecutive missing tokens on 100K examples from Brown corpus, using BERT-base with beam size 5 and number of missing tokens ranging from 2 to 5.}
    \label{tab:combined_metrics}
\end{table*}

\section{Experimental Setup}
\subsection{Models}
For English text data, we use the MLMs (\# params) BERT-base (110M), BERT-large (340M), DistilBERT (66M), and RoBERTa (125M) \cite{devlin_bert_2019, sanh2019distilbert, liu2019roberta}. For ancient texts, we use the MLMs Ithaca (49M) and Desformers (126M), two character-level BERT models trained on ancient Greek \cite{assael2022restoring, devaul_desformers}. For protein language modeling, we use ESM2 (8M) \cite{lin_evolutionary-scale_2023}.

\subsection{Metrics}
To assess text infilling results, we use top-$k$ accuracy, BLEU score,\footnote{When there are only $k<4$ tokens to infill, we employ BLEU-$k$.} and BERTScore \cite{Papineni_Roukos_Ward_Zhu_2001, Zhang_Kishore_Wu_Weinberger_Artzi_2020}. We do not compute perplexity, which, for arbitrary scoring schemes, requires calculation of an intractable normalization constant. In contrast, top-$k$ accuracy and BLEU are directly computable. Following \citet{Shen_Quach_Barzilay_Jaakkola_2020}, we use top-1 predictions of each infilling scheme and the ground truth span to compute BLEU score and BERTScore.

\subsection{Datasets}
In English, we perform infilling experiments on three datasets: the Brown corpus \cite{brown}, the Stanford Natural Language Inference dataset (SNLI) \cite{Bowman_Angeli_Potts_Manning_2015}, and the extreme summarization dataset (XSUM) \cite{narayan_dont_2018}. This allows us to additionally test varying amounts of context: the datasets have an average of 452, 19, and 28 tokens of context per example, respectively. For ancient language models, we perform infilling experiments on the Packhard Humanities Institute's database of ancient Greek inscriptions, collected by \citet{sommerschield2021iphi}. For protein language models, we perform infilling experiments on a subset of protein sequences from UniProt \cite{uniprot_universal_2008}.

%\begin{itemize}
%\itemsep0em 
%    \item The Brown corpus \cite{brown},
%    \item the Stanford Natural Language Inference dataset (SNLI) \cite{Bowman_Angeli_Potts_Manning_2015},
%    \item and the extreme summarization dataset (XSUM) \cite{narayan_dont_2018}.
%\end{itemize} 

%\begin{itemize}
%\itemsep0em 
%    \item For protein language models, we perform infilling experiments on UniProt \cite{uniprot_universal_2008}.
%    \item For ancient language models, we perform infilling experiments on the Packhard Humanities Institute's database of ancient Greek inscriptions\footnote{\url{https://inscriptions.packhum.org/}}, collected by \citet{sommerschield2021iphi}.
%\end{itemize}

Each experiment consists of selecting a random subset of $k$ contiguous indices from a test example and performing infilling according to HCB beam search. We run experiments with varying values of $k$, beam size, and pivot choices. Additionally, we take advantage of MLMs' non-autoregressive nature and infill tokens in order of highest MLM confidence (best-to-worst), following \citet{Schick_Schütze_2021} and \citet{assael2022restoring}.

\subsection{Baselines}
We compare HCB beam search (both left-to-right and best-to-worst) to a number of zero-shot infilling baselines. The first is the standard MLM beam search of \cref{alg:beam} with the ``Standard'' $f$, for which we also consider both left-to-right and best-to-worst beam searches.

We also compare our method to adapted versions of several popular sampling schemes including nucleus sampling \cite{holtzman2020curious} and sampling with temperature \cite{Ackley_Hinton_Sejnowski_1985}. Since such schemes are designed to \textit{sample}, instead of to \textit{search}, they inherently involve fewer forward passes through the model than beam search. We therefore specifically compare HCB beam search with beam size $B$ to these hybrid sampling-search schemes in which we sample and store in memory $B$ candidate samples for each token. This approach ensures that the same number of forward passes is used by each scheme, therefore resulting in the same time and space complexity of each method.

\subsection{Ablations}
To test that the estimated probability of the mask token $p([\M]|\cdot)$ is not simply random noise, as we found a pivot of mask tokens is crucial for the success of HCB beam search, we perform two ablations. 
\paragraph{Ablation 1 (Context Scramble):}\hypertarget{ablation1}
First, we test the hypothesis that $p([\M]|\mathbf{x})$ is sensitive to the given context $\mathbf{x}$. To do so, we track the values of the last 1,000 calls to $p([\M]|\mathbf{x}^{(i)})$ for the most recent contexts seen $(\mathbf{x}^{(1)},\hdots,\mathbf{x}^{(1,000)})$, and for the ablation, we replace the value of $p([\M]|\mathbf{x})$ in \cref{alg:beam} (``HCB'' $f$) with a random element of the previous 1,000 calls. 
\paragraph{Ablation 2 (Random Token Swap):}\hypertarget{ablation2}
Second, we test the hypothesis that $p([\M]|\mathbf{x})$ is sensitive to the input $[\M]$. To do so, we replace $[\M]$ with a completely random token $y$, and replace the computation of $p([\M]|\mathbf{x})$ in HCB beam search with $p(y|\mathbf{x})$.

\section{Results}\label{sec:results}

\subsection{English language models}

In our English language experiments with BERT-base, we observe a consistent relative ranking of methods: HCB Best-to-Worst $>$ HCB Left-to-Right $>$ Standard Best-to-Worst $>$ Standard Left-to-Right. \autoref{tab:combined_metrics} shows that this ranking persists across metrics (top-$k$ accuracy, BLEU, BERTScore F1) and number of missing tokens (2 through 5). In \autoref{tab:datasets} and \autoref{fig:eng_models}, we see the ranking additionally persists across datasets (Brown, SNLI, XSUM) and beam sizes (5 and 20), respectively.

We observe that BERTScore values are largely consistent across beam search methods (\autoref{tab:english_multi_model}), likely due to the relatively small number of tokens being infilled in each example, but HCB consistently outperforms standard beam search. We note too that standard left-to-right beam search outperforms our nucleus sampling-beam search hybrid, consistent with the findings of \cite{Shaham_Levy_2022}, as well as all other sampling-based methods.

Across models we note that the various beam search methods rank similarly when using BERT-large and DistilBERT, though RoBERTa stands out as a case where standard beam search outperforms HCB beam search. For a given model, we also find that the ranking of methods is consistent across values of $k$ for top-$k$ accuracy (\autoref{sec:app_B}, \autoref{fig:eng_models}).

\begin{table}[]
    \small
    \centering
    \begin{tabular}{|c||c|c|}
    \hline
    Model & BERT-base & RoBERTa \\
    \hline\hline
         \multirow{2}{*}{
         \begin{tabular}{cc}
             \multirow{2}{*}{\circled{$\mathbb{B}$}} & HCB \\
              & Standard 
         \end{tabular}
         } & \textbf{11.43} & 25.29 \\
          & 10.63 & \textbf{25.81} \\
         \hline
         \multirow{2}{*}{
         \begin{tabular}{cc}
             \multirow{2}{*}{\circled{$\mathbb{S}$}} & HCB \\
              & Standard 
         \end{tabular}
         } & \textbf{7.49} & 13.66 \\
          & 7.08 & \textbf{14.56} \\
         \hline
         \multirow{2}{*}{
         \begin{tabular}{cc}
             \multirow{2}{*}{\circled{$\mathbb{X}$}} & HCB \\
              & Standard
         \end{tabular}
         } & \textbf{9.86} & 23.21 \\
         & 9.18 & \textbf{24.52} \\
         \hline
    \end{tabular}
    \caption{Comparison of top-1 accuracy for HCB vs. standard beam search on the Brown \circled{$\mathbb{B}$}, SNLI \circled{$\mathbb{S}$} and XSUM \circled{$\mathbb{X}$} datasets, across models. We only show best-to-worst, since left-to-right was strictly worse in each case.}
    \label{tab:datasets}
\end{table}

\begin{table*}[]
    \small
    \centering
    \begin{tabular}{|c||c|c|c|c|}
    \hline
    Model & BERT-base & BERT-large & DistilBERT & RoBERTa \\
\hline\hline
HCB Left-to-Right & 10.86 & 11.92 & 6.12 & 24.74 \\
HCB Best-to-Worst & \textbf{11.43} & \textbf{12.58} & \textbf{6.44} & 25.29 \\
Standard Left-to-Right & 9.70 & 11.31 & 5.87 & 24.98 \\
Standard Best-to-Worst & 10.63 & 12.32 & 6.39 & \textbf{25.81} \\
\hline
Ablation 1 (Context Scramble) & 4.63 & 4.08 & 2.06 & 9.86 \\
Ablation 2 (Random Token Swap) & 9.71 & 11.16 & 5.23 & 17.61 \\
\hline
\end{tabular}
    \caption{Top-1 \% accuracy at infilling a random number (uniform between 2 and 5) of missing tokens on 100K examples from Brown corpus, across models, with beam size 5. For ablations, we show only best-to-worst results since they strictly outperformed left-to-right results.}
    \label{tab:english_multi_model}
\end{table*}

%\vspace{-5cm}

\begin{table*}[]
    \small
    \centering
    \begin{tabular}{|c||c|c|c||c|c|c||c|c|c|}
    \hline
    & \multicolumn{3}{c||}{Top-1 Accuracy (\%)} & \multicolumn{3}{c||}{Top-5 Accuracy (\%)} & \multicolumn{3}{c|}{Top-10 Accuracy (\%)} \\
    \hline
    \textit{Beam size} & 10 & 15 & 20 & 10 & 15 & 20 & 10 & 15 & 20 \\
    \hline
    Desformers HCB & 51.66 & 52.18 & 52.59 & \textbf{63.11} & \textbf{62.58} & \textbf{62.27} & \textbf{67.88} & \textbf{69.67} & \textbf{66.56} \\
    Desformers Standard & \textbf{55.21} & \textbf{56.08} & \textbf{56.62} & 60.79 & 60.80 & 60.98 & 64.88 & 66.29 & 63.43 \\
    \hline
    \hline  % Use double hline for separation or adjust spacing below
    & \multicolumn{3}{c||}{Top-1 Accuracy (\%)} & \multicolumn{3}{c||}{Top-10 Accuracy (\%)} & \multicolumn{3}{c|}{Top-20 Accuracy (\%)} \\
    \hline
    \textit{Number of missing tokens} & 5 & 6 & 7 & 5 & 6 & 7 & 5 & 6 & 7 \\
    \hline
    Ithaca HCB & \textbf{68.91} & \textbf{61.97} & \textbf{53.82} & \textbf{85.33} & \textbf{79.42} & \textbf{72.33} & \textbf{87.71} & 81.81 & 75.14 \\
    Ithaca Standard & 68.87 & 61.86 & 53.75 & 85.24 & 79.38 & 71.21 & 87.68 & \textbf{81.83} & \textbf{75.17} \\
    \hline
    \end{tabular}
    \caption{Top-k accuracy for infilling tasks on inscription datasets. (Top) Desformers top-k accuracy on 10K examples infilling two missing tokens using \textsc{[MASK]} pivot. (Bottom) Ithaca top-k accuracy on 6K examples infilling using beam size 20 and ``-'' pivot.}

    \label{tab:combined_beam}
\end{table*}

\subsection{Domain-specific language models}

\textit{Ancient texts.} In experiments infilling two missing characters with Desformers, standard beam search outperforms HCB beam search in top-1 accuracy, but HCB shows significant improvements for larger top-$k$ values (\autoref{tab:combined_beam}). However, HCB's performance with this model decreases relative to the standard method when infilling larger gaps.\footnote{On gap sizes larger than three characters, HCB is inferior to standard beam search across all experiments with this model.} With Ithaca, HCB beam search shows consistent improvement over standard beam search across different gap sizes (\autoref{tab:combined_beam}).

\noindent \textit{Protein sequences.} In experiments infilling between two and five missing amino acids in a protein sequence using ESM2, we find that HCB beam search shows comparable, albeit lower performance to standard beam search (\autoref{sec:app_B}, \autoref{tab:protein_results}). Thus, it appears that for protein sequences, the conditional independence assumption may hold better than for human language, leading to relatively strong standard beam search performance; this is further explored in \Cref{sec:discussion}.  

\begin{figure}
    \centering
    \includegraphics[width=0.95\linewidth]{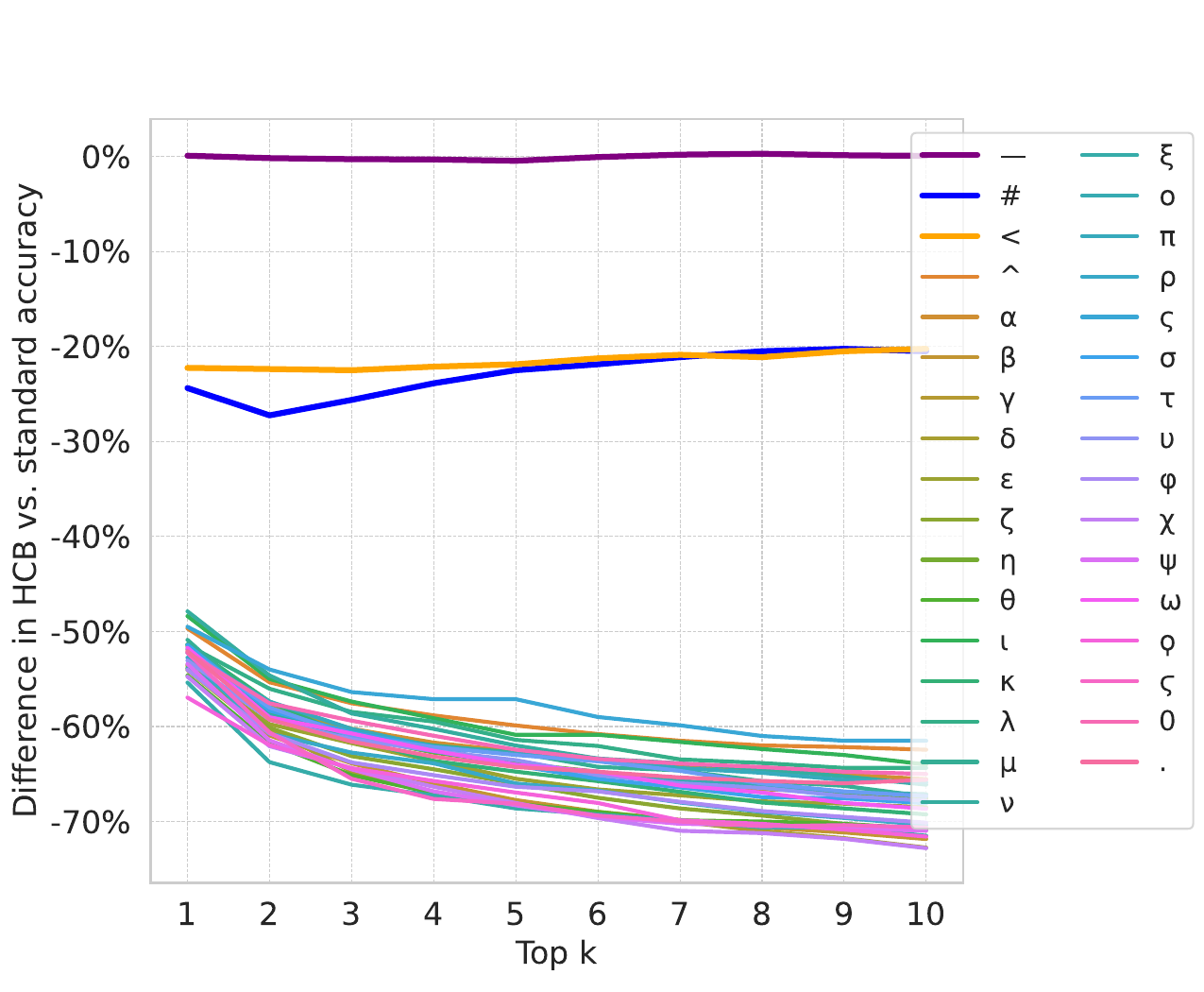}
    \caption{Difference in Ithaca infilling accuracy across all choices of pivots.}
    \label{fig:pivots}
\end{figure}

\subsection{Pivot design}

To explore how performance of HCB beam search depends on the choice of pivot, we extensively test all pivots for seven-character infills with a beam size of 20 for Ithaca. Testing every possible pivot choice is made possible by Ithaca's small 34-token vocabulary. With over 800 trials per pivot, we observe that the special token ``-'', which corresponds to a missing character, is clearly the best performing pivot choice, achieving accuracy similar to standard beam search. Two other special tokens ``\#'' and ``$<$'' perform about $20\%$ worse as HCB pivots, and the 31 remaining tokens perform over $50\%$ worse than standard beam search (\autoref{fig:pivots}). We also exhaustively test all pivots for ESM2, observing that the best choice of pivot appears to be the special token ``[CLS]'' (\autoref{sec:app_B}, \autoref{fig:prot_pivot}), though this optimal pivot choice is much less clear than for Ithaca. 

Notably, Ithaca's architecture does not include the mask token in its output probability distribution; thus, the missing character token, which is scattered randomly throughout train and test examples, is the token which remains most in-distribution when injected as a pivot into random positions. Similarly, for ESM2, we observe that some special tokens perform noticeably better as pivots than regular tokens. These results suggest that special tokens, which do not add disruptive context, could be better candidates for pivots. 

\subsection{Ablations}\label{subsec:ablationresults}

\hyperlink{ablation1}{Ablation 1} dramatically worsens the performance of HCB beam search, as we show in \Cref{tab:english_multi_model}. \hyperlink{ablation2}{Ablation 2} also deteriorates the performance of HCB beam search, although less so than \hyperlink{ablation1}{Ablation 1}. We conclude that $p([\M]|\mathbf{x})$ is sensitive to both the context $\mathbf{x}$ and input token $[\M]$, and we hypothesize that this contributes to the success of HCB beam search. Since results are far worse when we randomly replace $\mathbf{x}$ than when we randomly replace $[\M]$, we hypothesize that the context $\mathbf{x}$ is a more important factor to the success of HCB beam search.

\section{Discussion}\label{sec:discussion}
In \Cref{sec:results}, we see that HCB beam search achieves comparable performance to standard beam search across languages and that superiority of either method is usually model-dependent, rather than data-dependent: for instance, HCB beam search consistently outperforms standard beam search with BERT-base, but consistently underperforms with RoBERTa across datasets as shown by \autoref{tab:datasets}.

One plausible hypothesis is that each beam search variant performs well when certain assumptions hold, and the validity of these assumptions is determined by the training procedure. For instance, given a probability distribution $p$ on fixed-length sequences $\mathbf{x}$, the claim that
\[
\log p(\mathbf{x}) = \sum_{i=1}^n \log p(x_i|\mathbf{x}_{:i}, \mathbf{[\M]}_{i:})
\]
holds for all $n$ and all $\mathbf{x}$ is equivalent to the claim that $x_i$ is conditionally independent of $\mathbf{[\M]}_{i:}$ given $\mathbf{x}_{:i}$. Therefore, standard beam search considers the true joint probability $p(\mathbf{x})$ exactly when this conditional independence holds.

% For this reason, we might expect standard beam search to function well exactly when this conditional independence holds. BERT's training procedure is known to encourage this conditional independence \cite{devlin_bert_2019}, and so we might expect independence to be a more reasonable assumption for more finely optimized BERT models, such as RoBERTa. Indeed, this hypothesis is consistent with \autoref{tab:datasets}, which finds that standard beam search outperforms HCB beam search over all datasets considered when RoBERTa is used, but not when BERT-base is used. We hypothesize this may also explain why standard beam search outperforms HCB beam search for ESM2. 

The following theorem shows that any set of learned conditional distributions which achieve minimal loss on the MLM objective satisfy the conditional independence assumption:

\begin{thm}\label{thm:proper}
    Suppose that $\mathbf{x}^1,\hdots,\mathbf{x}^m$ are sequences of length $n$, and let $\mathbb{P}$ be their empirical distribution. Let $\{Z_{ij}\}_{i\leq m,j\leq n}$ be auxiliary i.i.d. Bernoulli random variables with parameter $p>0$. For each ground truth sequence $\mathbf{x}^i$, define a ``partially masked'' sequence $\mathbf{y}^i$ such that
    \[
    y^i_j = \begin{cases}
        x^i_j \quad & Z_{ij} = 0 \\
        [\M] \quad & Z_{ij} = 1
    \end{cases}
    \]
    where $[\M]$ is an arbitrary symbol not contained in the alphabet over which $\mathbb{P}$ has support.
    
    Consider the (standard) MLM loss function\footnote{In addition to masking random tokens, some MLM training procedures (e.g. BERT) include random replacement of tokens. This property can be included by simply including all such random transformations as elements of the training dataset $\mathbf{x}^1,\hdots,\mathbf{x}^m$.} on sets of conditional distributions $p(x_i|\mathbf{x}_{-i})$:
    \[
    L(p) = {\mathbb{E}_{Z} }\left[\sum_{i,j:Z_{ij}=1} -\log p(x^i_j | \mathbf{y}^i_{-j})\right]
    \]
    If the set of learned conditionals minimizes the training loss $L(p)$, then $p$ satisfies the following conditional independence law:
    \[
    p(x_j|\mathbf{x}_{:j}, \mathbf{[\M]}_{j+1:}) = p(x_j|\mathbf{x}_{:j});
    \]
    and the conditional distributions $p(x_j|\mathbf{x}_{-j})$ are compatible.
\end{thm}
\begin{proof} See \autoref{sec:proof}.
\end{proof}

To explore this hypothesis, we trained two BERT models from scratch under different training conditions which might affect the validity of the conditional independence assumption. Specifically, one model was trained with dynamic masking, where the data is re-masked every epoch, and the other with static masking, where the data is masked only once at the start of training. After assessing a large number of checkpoints across both models, we do not observe HCB beam search becoming significantly less effective compared to standard beam search as training progressed. As performance appears to be model-specific, we recommend that researchers using beam search evaluate HCB beam search in their setting, as it requires minimal modification from standard beam search.

\section{Conclusion}

In this work, we develop and apply theoretically-sound methods to use pretrained MLMs for text infilling, a task with important applications spanning a wide variety of domains. We clarify the conditions under which it is theoretically sound to perform text infilling with MLMs using standard beam search. We introduce HCB beam search as a probabilistically-justified modification with no additional computational complexity and demonstrate its superiority to standard beam search with a suite of models. Future work exploring HCB beam search can help further elucidate the contexts in which HCB beam search is beneficial, especially in low-data or compute settings. 

\section{Limitations}

One limitation is that we do not explicitly know whether conditional independence assumptions hold for a given model, despite hypothesized heuristics. Another is the lack of intuition regarding the use of $\textsc{[mask]}$ as a pivot -- future work should more rigorously investigate optimal pivot choices and the interpretability of these entities. 

% Additionally, a further understanding of the objective and training conditions under which an MLM satisfies conditional independence assumptions is also needed. 

Moreover, the proposed HCB beam search relies on the presence of some token in a given model's output distribution which flexibly behaves as in-distribution; while some MLMs maybe include special characters in their output layers, others may not. A potential risk of our method is its ability to use MLMs outside of their original contexts, which could have unforeseen consequences and unpredictable behavior. 

Finally, our proposed method only applies to settings with a fixed number of missing tokens as opposed to the more general setting addressed by others \citep{Shen_Quach_Barzilay_Jaakkola_2020}. Nonetheless, we find this is a realistic assumption in some important applications, such as ancient text restoration and protein engineering. In the case of damaged inscriptions, domain experts posit an estimated number of missing characters based on physical distance \citep{bruun2014oxford}. In protein engineering, a reasonable constraint is maintaining a fixed protein sequence length while only modifying amino acids in functional regions. In both areas, the fixed length assumption has been used in cutting-edge models \citep{assael2022restoring, wang2022scaffolding}.

% Anonymizing for now... 
\section*{Acknowledgments}
We are grateful to Stephen Bates, Barbara Graziosi, Johannes Haubold, Lucas Torroba Hennigen, Yoon Kim,  Karthik Narasimhan, Indu Panigrahi, Henrique Schechter Vera, and Richard Zhu for helpful discussions and feedback on this work. 

This work is supported by NIGMS T32GM144273 (A.S.). The content is solely the responsibility of the authors and does not necessarily represent the official views of the National Institute of General Medical Sciences or the National Institutes of Health.

\clearpage

\newpage

\bibliographystyle{acl_natbib}
\bibliography{anthology, acl2021}

\begin{thebibliography}{39}
\expandafter\ifx\csname natexlab\endcsname\relax\def\natexlab#1{#1}\fi

\bibitem[{Ackley et~al.(1985)Ackley, Hinton, and Sejnowski}]{Ackley_Hinton_Sejnowski_1985}
David~H. Ackley, Geoffrey~E. Hinton, and Terrence~J. Sejnowski. 1985.
\newblock \href {https://doi.org/10.1207/s15516709cog0901_7} {A learning algorithm for boltzmann machines*}.
\newblock \emph{Cognitive Science}, 9(1):147–169.

\bibitem[{Arnold and Gokhale(1998)}]{Arnold_Gokhale_1998}
Barry~C. Arnold and D.~V. Gokhale. 1998.
\newblock \href {https://doi.org/10.1007/BF02565119} {Distributions most nearly compatible with given families of conditional distributions}.
\newblock \emph{Test}, 7(2):377–390.

\bibitem[{Assael et~al.(2022)Assael, Sommerschield, Shillingford, Bordbar, Pavlopoulos, Chatzipanagiotou, Androutsopoulos, Prag, and de~Freitas}]{assael2022restoring}
Yannis Assael, Thea Sommerschield, Brendan Shillingford, Mahyar Bordbar, John Pavlopoulos, Marita Chatzipanagiotou, Ion Androutsopoulos, Jonathan Prag, and Nando de~Freitas. 2022.
\newblock Restoring and attributing ancient texts using deep neural networks.
\newblock \emph{Nature}, 603(7900):280--283.

\bibitem[{Bamman and Burns(2020)}]{Bamman_Burns_2020}
David Bamman and Patrick~J. Burns. 2020.
\newblock Latin bert: A contextual language model for classical philology.

\bibitem[{Besag(1974)}]{besag_spatial_1974}
Julian Besag. 1974.
\newblock \href {https://www.jstor.org/stable/2984812} {Spatial {Interaction} and the {Statistical} {Analysis} of {Lattice} {Systems}}.
\newblock \emph{Journal of the Royal Statistical Society. Series B (Methodological)}, 36(2):192--236.
\newblock Publisher: [Royal Statistical Society, Wiley].

\bibitem[{Bowman et~al.(2015)Bowman, Angeli, Potts, and Manning}]{Bowman_Angeli_Potts_Manning_2015}
Samuel~R. Bowman, Gabor Angeli, Christopher Potts, and Christopher~D. Manning. 2015.
\newblock \href {https://doi.org/10.18653/v1/D15-1075} {A large annotated corpus for learning natural language inference}.
\newblock In \emph{Proceedings of the 2015 Conference on Empirical Methods in Natural Language Processing}, page 632–642, Lisbon, Portugal. Association for Computational Linguistics.

\bibitem[{Bruun and Edmondson(2014)}]{bruun2014oxford}
Christer Bruun and Jonathan Edmondson. 2014.
\newblock \emph{The Oxford handbook of Roman epigraphy}.
\newblock Oxford University Press.

\bibitem[{Cowen-Breen et~al.(2023{\natexlab{a}})Cowen-Breen, Brooks, Haubold, and Graziosi}]{cowen-breen_detecting_2023}
Charlie Cowen-Breen, Creston Brooks, Johannes Haubold, and Barbara Graziosi. 2023{\natexlab{a}}.
\newblock Logion: Machine-learning based detection and correction of textual errors in greek philology.
\newblock In \emph{Ancient Language Processing}.

\bibitem[{Cowen-Breen et~al.(2023{\natexlab{b}})Cowen-Breen, Brooks, Haubold, and Graziosi}]{cowen-breen_logion_2023}
Charlie Cowen-Breen, Creston Brooks, Johannes Haubold, and Barbara Graziosi. 2023{\natexlab{b}}.
\newblock \href {https://doi.org/10.48550/arXiv.2305.01099} {Logion: {Machine} {Learning} for {Greek} {Philology}}.
\newblock ArXiv:2305.01099 [cs].

\bibitem[{DeVaul(2023)}]{devaul_desformers}
Desmond DeVaul. 2023.
\newblock Desformers.
\newblock \url{https://huggingface.co/ddevaul/desformers}.

\bibitem[{Devlin et~al.(2019)Devlin, Chang, Lee, and Toutanova}]{devlin_bert_2019}
Jacob Devlin, Ming-Wei Chang, Kenton Lee, and Kristina Toutanova. 2019.
\newblock \href {https://doi.org/10.48550/arXiv.1810.04805} {{BERT}: {Pre}-training of {Deep} {Bidirectional} {Transformers} for {Language} {Understanding}}.
\newblock ArXiv:1810.04805 [cs].

\bibitem[{Donahue et~al.(2020)Donahue, Lee, and Liang}]{Donahue_Lee_Liang_2020}
Chris Donahue, Mina Lee, and Percy Liang. 2020.
\newblock \href {https://doi.org/10.18653/v1/2020.acl-main.225} {Enabling language models to fill in the blanks}.
\newblock In \emph{Proceedings of the 58th Annual Meeting of the Association for Computational Linguistics}, page 2492–2501, Online. Association for Computational Linguistics.

\bibitem[{Graziosi et~al.(2023)Graziosi, Haubold, Cowen-Breen, and Brooks}]{Graziosi_Haubold_Cowen-Breen_Brooks_2023}
Barbara Graziosi, Johannes Haubold, Charlie Cowen-Breen, and Creston Brooks. 2023.
\newblock \href {https://doi.org/10.1353/apa.2023.a901022} {Machine learning and the future of philology: A case study}.
\newblock \emph{TAPA}, 153(1):253–284.

\bibitem[{He et~al.(2023)He, Sun, Tang, Wang, Huang, and Qiu}]{He_Sun_Tang_Wang_Huang_Qiu_2023}
Zhengfu He, Tianxiang Sun, Qiong Tang, Kuanning Wang, Xuanjing Huang, and Xipeng Qiu. 2023.
\newblock \href {https://doi.org/10.18653/v1/2023.acl-long.248} {Diffusionbert: Improving generative masked language models with diffusion models}.
\newblock In \emph{Proceedings of the 61st Annual Meeting of the Association for Computational Linguistics (Volume 1: Long Papers)}, page 4521–4534, Toronto, Canada. Association for Computational Linguistics.

\bibitem[{Hennigen and Kim(2023)}]{hennigen_deriving_2023}
Lucas~Torroba Hennigen and Yoon Kim. 2023.
\newblock \href {https://doi.org/10.48550/arXiv.2305.15501} {Deriving {Language} {Models} from {Masked} {Language} {Models}}.
\newblock ArXiv:2305.15501 [cs].

\bibitem[{Holtzman et~al.(2020)Holtzman, Buys, Du, Forbes, and Choi}]{holtzman2020curious}
Ari Holtzman, Jan Buys, Li~Du, Maxwell Forbes, and Yejin Choi. 2020.
\newblock \href {http://arxiv.org/abs/1904.09751} {The curious case of neural text degeneration}.

\bibitem[{Hu et~al.(2021)Hu, Shen, Wallis, Allen-Zhu, Li, Wang, Wang, and Chen}]{hu_lora_2021}
Edward~J. Hu, Yelong Shen, Phillip Wallis, Zeyuan Allen-Zhu, Yuanzhi Li, Shean Wang, Lu~Wang, and Weizhu Chen. 2021.
\newblock \href {http://arxiv.org/abs/2106.09685} {{LoRA}: {Low}-{Rank} {Adaptation} of {Large} {Language} {Models}}.
\newblock ArXiv:2106.09685 [cs].

\bibitem[{Ippolito et~al.(2019)Ippolito, Grangier, Callison-Burch, and Eck}]{Ippolito_Grangier_Callison-Burch_Eck_2019}
Daphne Ippolito, David Grangier, Chris Callison-Burch, and Douglas Eck. 2019.
\newblock \href {https://doi.org/10.18653/v1/W19-2405} {Unsupervised hierarchical story infilling}.
\newblock In \emph{Proceedings of the First Workshop on Narrative Understanding}, page 37–43, Minneapolis, Minnesota. Association for Computational Linguistics.

\bibitem[{Lin et~al.(2023)Lin, Akin, Rao, Hie, Zhu, Lu, Smetanin, Verkuil, Kabeli, Shmueli, dos Santos~Costa, Fazel-Zarandi, Sercu, Candido, and Rives}]{lin_evolutionary-scale_2023}
Zeming Lin, Halil Akin, Roshan Rao, Brian Hie, Zhongkai Zhu, Wenting Lu, Nikita Smetanin, Robert Verkuil, Ori Kabeli, Yaniv Shmueli, Allan dos Santos~Costa, Maryam Fazel-Zarandi, Tom Sercu, Salvatore Candido, and Alexander Rives. 2023.
\newblock \href {https://doi.org/10.1126/science.ade2574} {Evolutionary-scale prediction of atomic-level protein structure with a language model}.
\newblock \emph{Science}, 379(6637):1123--1130.
\newblock Publisher: American Association for the Advancement of Science.

\bibitem[{Liu et~al.(2019)Liu, Ott, Goyal, Du, Joshi, Chen, Levy, Lewis, Zettlemoyer, and Stoyanov}]{liu2019roberta}
Yinhan Liu, Myle Ott, Naman Goyal, Jingfei Du, Mandar Joshi, Danqi Chen, Omer Levy, Mike Lewis, Luke Zettlemoyer, and Veselin Stoyanov. 2019.
\newblock Roberta: A robustly optimized bert pretraining approach.
\newblock \emph{arXiv preprint arXiv:1907.11692}.

\bibitem[{Narayan et~al.(2018)Narayan, Cohen, and Lapata}]{narayan_dont_2018}
Shashi Narayan, Shay~B. Cohen, and Mirella Lapata. 2018.
\newblock \href {http://arxiv.org/abs/1808.08745} {Don't {Give} {Me} the {Details}, {Just} the {Summary}! {Topic}-{Aware} {Convolutional} {Neural} {Networks} for {Extreme} {Summarization}}.
\newblock ArXiv:1808.08745 [cs].

\bibitem[{Nelson~Francis and Kucera(1979)}]{brown}
W.~Nelson~Francis and Henry Kucera. 1979.
\newblock \emph{Brown corpus manual}.
\newblock Brown University.

\bibitem[{Papineni et~al.(2001)Papineni, Roukos, Ward, and Zhu}]{Papineni_Roukos_Ward_Zhu_2001}
Kishore Papineni, Salim Roukos, Todd Ward, and Wei-Jing Zhu. 2001.
\newblock \href {https://doi.org/10.3115/1073083.1073135} {Bleu: a method for automatic evaluation of machine translation}.
\newblock In \emph{Proceedings of the 40th Annual Meeting on Association for Computational Linguistics - ACL ’02}, page 311, Philadelphia, Pennsylvania. Association for Computational Linguistics.

\bibitem[{Riemenschneider and Frank(2023)}]{riemenschneider2023exploring}
Frederick Riemenschneider and Anette Frank. 2023.
\newblock Exploring large language models for classical philology.
\newblock \emph{arXiv preprint arXiv:2305.13698}.

\bibitem[{Ruffolo and Madani(2024)}]{Ruffolo_Madani_2024}
Jeffrey~A. Ruffolo and Ali Madani. 2024.
\newblock \href {https://doi.org/10.1038/s41587-024-02123-4} {Designing proteins with language models}.
\newblock \emph{Nature Biotechnology}, 42(2):200–202.

\bibitem[{Sanh et~al.(2019)Sanh, Debut, Chaumond, and Wolf}]{sanh2019distilbert}
Victor Sanh, Lysandre Debut, Julien Chaumond, and Thomas Wolf. 2019.
\newblock Distilbert, a distilled version of bert: smaller, faster, cheaper and lighter.
\newblock \emph{arXiv preprint arXiv:1910.01108}.

\bibitem[{Schick and Schütze(2021)}]{Schick_Schütze_2021}
Timo Schick and Hinrich Schütze. 2021.
\newblock \href {https://doi.org/10.18653/v1/2021.naacl-main.185} {It’s not just size that matters: Small language models are also few-shot learners}.
\newblock In \emph{Proceedings of the 2021 Conference of the North American Chapter of the Association for Computational Linguistics: Human Language Technologies}, page 2339–2352, Online. Association for Computational Linguistics.

\bibitem[{Shaham and Levy(2022)}]{Shaham_Levy_2022}
Uri Shaham and Omer Levy. 2022.
\newblock \href {https://doi.org/10.18653/v1/2022.insights-1.5} {What do you get when you cross beam search with nucleus sampling?}
\newblock In \emph{Proceedings of the Third Workshop on Insights from Negative Results in NLP}, page 38–45, Dublin, Ireland. Association for Computational Linguistics.

\bibitem[{Shen et~al.(2020)Shen, Quach, Barzilay, and Jaakkola}]{Shen_Quach_Barzilay_Jaakkola_2020}
Tianxiao Shen, Victor Quach, Regina Barzilay, and Tommi Jaakkola. 2020.
\newblock \href {https://doi.org/10.18653/v1/2020.emnlp-main.420} {Blank language models}.
\newblock In \emph{Proceedings of the 2020 Conference on Empirical Methods in Natural Language Processing (EMNLP)}, page 5186–5198, Online. Association for Computational Linguistics.

\bibitem[{Sledzieski et~al.(2023)Sledzieski, Kshirsagar, Baek, Berger, Dodhia, and Ferres}]{sledzieski_democratizing_2023}
Samuel Sledzieski, Meghana Kshirsagar, Minkyung Baek, Bonnie Berger, Rahul Dodhia, and Juan~Lavista Ferres. 2023.
\newblock \href {https://doi.org/10.1101/2023.11.09.566187} {Democratizing {Protein} {Language} {Models} with {Parameter}-{Efficient} {Fine}-{Tuning}}.
\newblock Pages: 2023.11.09.566187 Section: New Results.

\bibitem[{Sommerschield* et~al.(2021)Sommerschield*, Assael*, Shillingford, Bordbar, Pavlopoulos, Chatzipanagiotou, Androutsopoulos, Prag, and de~Freitas}]{sommerschield2021iphi}
Thea Sommerschield*, Yannis Assael*, Brendan Shillingford, Mahyar Bordbar, John Pavlopoulos, Marita Chatzipanagiotou, Ion Androutsopoulos, Jonathan Prag, and Nando de~Freitas. 2021.
\newblock {I.PHI} dataset: ancient greek inscriptions.
\newblock \url{https://github.com/sommerschield/iphi}.

\bibitem[{Sun et~al.(2017)Sun, Lee, and Batra}]{Sun_Lee_Batra_2017}
Qing Sun, Stefan Lee, and Dhruv Batra. 2017.
\newblock \href {https://doi.org/10.1109/CVPR.2017.763} {Bidirectional beam search: Forward-backward inference in neural sequence models for fill-in-the-blank image captioning}.
\newblock In \emph{2017 IEEE Conference on Computer Vision and Pattern Recognition (CVPR)}, page 7215–7223, Honolulu, HI. IEEE.

\bibitem[{Tibshirani, 2023()}]{scoring}
Tibshirani, 2023.
\newblock Stat 241B lecture notes.
\newblock \href {https://www.stat.berkeley.edu/~ryantibs/statlearn-s23/lectures/calibration.pdf} {[link]}.

\bibitem[{UniProt(2008)}]{uniprot_universal_2008}
UniProt. 2008.
\newblock \href {https://pubmed.ncbi.nlm.nih.gov/18045787/} {The universal protein resource ({UniProt}) - {PubMed}}.

\bibitem[{Wang and Cho(2019)}]{wang_bert_2019}
Alex Wang and Kyunghyun Cho. 2019.
\newblock \href {https://doi.org/10.18653/v1/W19-2304} {{BERT} has a {Mouth}, and {It} {Must} {Speak}: {BERT} as a {Markov} {Random} {Field} {Language} {Model}}.
\newblock In \emph{Proceedings of the {Workshop} on {Methods} for {Optimizing} and {Evaluating} {Neural} {Language} {Generation}}, pages 30--36, Minneapolis, Minnesota. Association for Computational Linguistics.

\bibitem[{Wang et~al.(2022)Wang, Lisanza, Juergens, Tischer, Watson, Castro, Ragotte, Saragovi, Milles, Baek et~al.}]{wang2022scaffolding}
Jue Wang, Sidney Lisanza, David Juergens, Doug Tischer, Joseph~L Watson, Karla~M Castro, Robert Ragotte, Amijai Saragovi, Lukas~F Milles, Minkyung Baek, et~al. 2022.
\newblock Scaffolding protein functional sites using deep learning.
\newblock \emph{Science}, 377(6604):387--394.

\bibitem[{Young and You(2023)}]{young2023inconsistencies}
Tom Young and Yang You. 2023.
\newblock \href {http://arxiv.org/abs/2301.00068} {On the inconsistencies of conditionals learned by masked language models}.

\bibitem[{Zhang et~al.(2020)Zhang, Kishore, Wu, Weinberger, and Artzi}]{Zhang_Kishore_Wu_Weinberger_Artzi_2020}
Tianyi Zhang, Varsha Kishore, Felix Wu, Kilian~Q. Weinberger, and Yoav Artzi. 2020.
\newblock \href {https://doi.org/10.48550/arXiv.1904.09675} {Bertscore: Evaluating text generation with bert}.
\newblock (arXiv:1904.09675).
\newblock ArXiv:1904.09675 [cs].

\bibitem[{Zhu et~al.(2019)Zhu, Hu, and Xing}]{zhu2019text}
Wanrong Zhu, Zhiting Hu, and Eric Xing. 2019.
\newblock Text infilling.
\newblock \emph{arXiv preprint arXiv:1901.00158}.

\end{thebibliography}

\clearpage

% \twocolumn[ % Centered title in two-column format
%   \begin{@twocolumnfalse}
%   \begin{center}
%     \Large \textbf{Appendix} % Large, bold appendix title
%     \vspace{1cm} % Space after the title
%   \end{center}
%   \end{@twocolumnfalse}
% ]
\onecolumn
\begin{center}
\Large \textbf{Appendix}
\end{center}

\appendix

\section{Proofs}\label{sec:proof}

\begin{proof}[Proof of \Cref{thm:proper}]
Log-loss is a strictly proper scoring function \cite{scoring}, in that it is uniquely minimized when $p=\mathbb{P}.$\footnote{For instance, see the following blog post on a consistency theorem for BERT:  \url{https://machinethoughts.wordpress.com/2019/07/14/a-consistency-theorem-for-bert/}} From the fact that $\mathbb{P}$ is a probability distribution, it immediately follows that the conditionals $p(x_j|\mathbf{x}_{-j})$ are compatible. On the other hand, consider the events
    \[
    p(x^i_j = x | \mathbf{y}^i_{:j} = \mathbf{z}_{:j}, \mathbf{y}^i_{j+1:} = \mathbf{[\M]}_{j+1:})
    \]
    for arbitrary $x$ and $\mathbf{z}$. As the event $\{\mathbf{y}^i_k = [\M]\}$ occurs if and only if $Z_{ik}=1$, and thus with probability $p$ independently of $\mathbf{x}$ for all $i\in[m],j\in[n]$, the conditional independence claim 
    \[
    p(x^i_j | \mathbf{x}^i_{:j}, \mathbf{[M]}_{j+1:}) = p(x^i_j | \mathbf{x}^i_{:j})
    \]
    follows.
\end{proof}
\vspace{5mm}
\begin{proof}[Proof that compatibility and conditional independence imply validity of \autoref{eqn:approxbeam}]
By compatibility, there exists a joint distribution $p(\mathbf{x})$ whose conditional distributions are equal to those learned by the MLM. By the chain rule and the assumption of conditional independence,
\begin{align*}
    \log p(\mathbf{x}) &= \sum_{i=1}^n \log p(x_i |\mathbf{x}_{<i}) \\
    &= \sum_{i=1}^n \log p(x_i | \mathbf{x}_{:i}, \mathbf{[M]}_{i:})
\end{align*}
On the other hand, if
\[
\log p(\mathbf{x}) = \sum_{i=1}^n \log p(x_i | \mathbf{x}_{:i}, \mathbf{[M]}_{i:})
\]
for all $n$, then we recover the conditional independence law
\[
p(x_j|\mathbf{x}_{:j}, [\M]_{j+1:}) = p(x_j|\mathbf{x}_{:j})
\]
by induction on $n$.
\end{proof}

\section{Algorithms} \label{sec:app_algorithms}

\begin{algorithm}
\small
\caption{Autoregressive beam search. Given a beam size $B>0$, returns a collection of generated sequences $S$ of length $n$.}\label{alg:auto}
\begin{algorithmic}
\State Initialize $S = \{(0,\emptyset)\}$
\For{$i\in\{1,\hdots,n\}$}
\For{$(\ell,(x_1,\hdots,x_{i-1}))\in S$}
\State $f(\cdot) \gets \log p(\cdot | {x_1,\hdots,x_{i-1}})$ (1 forward pass.)
\State Append to $S$: $(\ell + f(x), (x_1,\hdots,x))$ for every $x$.
\EndFor
\State $S \gets$ $\{$the $B$ sequences $(x_1,\hdots,x_i)$ of $S$ w/ highest $\ell\}$
\EndFor \\
\Return S
\end{algorithmic}
\end{algorithm}

\section{Additional figures} \label{sec:app_B}

\begin{table*}[h!]
    \centering
    \begin{tabular}{|c|c|c|c|c|}\hline
    Model & BERT-base & BERT-large & DistilBERT & RoBERTa \\
\hline
HCB Left-to-Right & 15.00 & 16.24 & 9.05 & 31.85 \\
HCB Best-to-Worst & \textbf{15.72} & \textbf{16.89} & 9.62 & 32.47 \\
Standard Left-to-Right & 14.02 & 15.95 & 9.03 & 32.40 \\
Standard Best-to-Worst & 14.87 & 16.72 & \textbf{9.65} & \textbf{33.15} \\
\hline
Pure sampling & 9.69 & 11.56 & 4.67 & 22.76 \\
Sampling with $T=0.25$ & 12.33 & 14.14 & 7.69 & 28.57 \\
Sampling with $T=0.5$ & 11.93 & 13.67 & 7.11 & 27.63 \\
Sampling with $T=0.75$ & 10.98 & 12.86 & 6.05 & 25.84 \\
Nucleus sampling, $p=0.9$ & 9.73 & 11.30 & 5.02 & 24.56 \\
\hline
\end{tabular}
    \caption{Top-5 accuracy at infilling a random number (uniform between 2 and 5) of missing tokens on 100K examples from Brown corpus, across models, with beam size 5.}
    \label{tab:my_label}
\end{table*}

% \begin{table*}[h!]
%     \small
%     \centering
%     \begin{tabular}{|c||c|c|c|c|}
%     \hline
%     Model & BERT-base & BERT-large & DistilBERT & RoBERTa \\
% \hline\hline
% HCB Left-to-Right & 10.86 & 11.92 & 6.12 & 24.74 \\
% HCB Best-to-Worst & \textbf{11.43} & \textbf{12.58} & \textbf{6.44} & 25.29 \\
% Standard Left-to-Right & 9.70 & 11.31 & 5.87 & 24.98 \\
% Standard Best-to-Worst & 10.63 & 12.32 & 6.39 & \textbf{25.81} \\
% \hline
% Pure sampling & 8.09 & 9.57 & 4.00 & 20.72 \\
% Sampling w/ $T=0.25$ & 9.41 & 10.97 & 5.66 & 24.59 \\
% Sampling w/ $T=0.5$ & 9.25 & 10.74 & 5.40 & 24.03 \\
% Sampling w/ $T=0.75$ & 8.79 & 10.31 & 4.85 & 22.89 \\
% Nucleus sampling, $p=0.9$ & 8.45 & 9.99 & 4.30 & 22.57 \\
% \hline
% Ablation 1 (Context Scramble) & 4.63 & 4.08 & 2.06 & 9.86 \\
% Ablation 2 (Random Token Swap) & 9.71 & 11.16 & 5.23 & 17.61 \\
% \hline
% \end{tabular}
%     \caption{Top-1 accuracy at infilling a random number (uniform between 2 and 5) of missing tokens on 100K examples from Brown corpus, across models, with beam size 5. For ablations, we show only best-to-worst results since they strictly outperformed left-to-right results.}
%     \label{tab:my_label}
% \end{table*}

\begin{table*}[h!]
    \centering
    \begin{tabular}{|c||c|c|c|c||c|c|c|c|}
    \hline
    & \multicolumn{4}{c||}{Top-1 Accuracy (\%)} & \multicolumn{4}{c|}{Top-5 Accuracy (\%)} \\
    \hline
    \hline
   Number of missing tokens & 2 & 3 & 4 & 5 & 2 & 3 & 4 & 5 \\
\hline
HCB Left-to-Right & 6.16 & 2.12 & 0.91 & 0.46 & 16.79 & 5.80 & 2.41 & 1.16 \\
HCB Best-to-Worst & 6.36 & 2.30 & 1.02 & 0.53 & 16.95 & 6.03 & 2.53 & 1.19 \\
Standard Left-to-Right & 6.57 & 2.49 & 1.29 & 0.71 & 17.64 & 6.51 & 2.92 & 1.50 \\
Standard Best-to-Worst & \textbf{6.65} & \textbf{2.58} & \textbf{1.31} & \textbf{0.76} & \textbf{17.76} & \textbf{6.70} & \textbf{3.09} & \textbf{1.67} \\
\hline
\end{tabular}
    \caption{Top-$k$ accuracy at infilling consecutive missing tokens on 10K examples from UniProt dataset, using ESM2 with beam size 5 and number of missing tokens ranging from 2 to 5.}
    \label{tab:protein_results}
\end{table*}

\begin{figure*}[h!]
    \centering \includegraphics[width=1.0\textwidth]{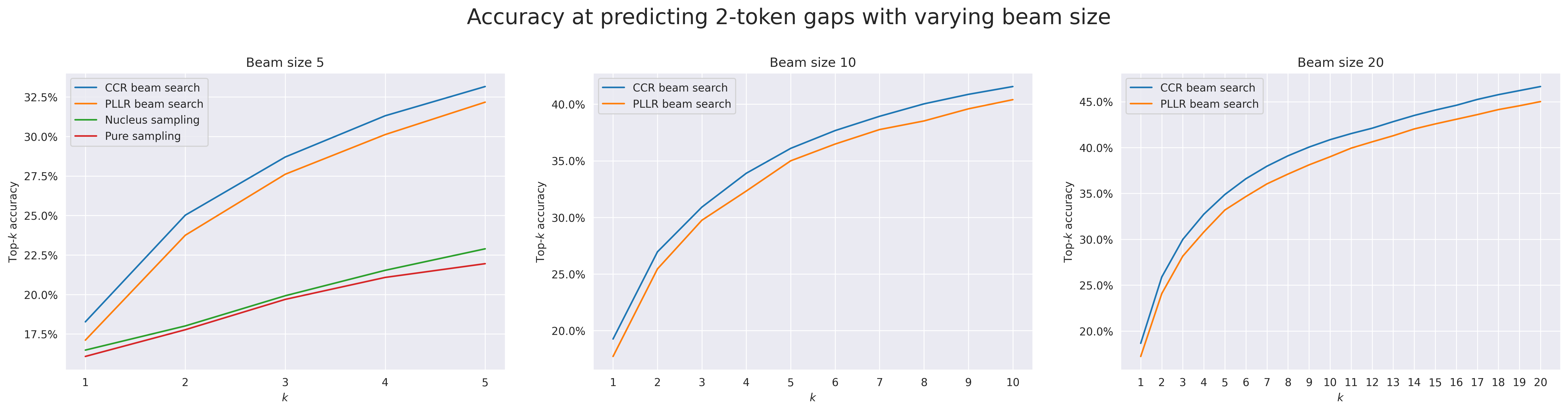}
    \caption{Comparison of HCB beam search with standard beam search, nucleus sampling, and pure sampling. Evaluated on 10,000 examples from the SNLI dataset. When comparing nucleus sampling to beam search with beam size $B$, we draw $B$ samples for a fair comparison.}
    \label{fig:varyingbeam}
\end{figure*}

% \begin{figure*}[h!]
%     \centering \includegraphics[width=1.0\textwidth]{Unknown-4.png}
%     \caption{Comparison of HCB beam search with standard beam search. Evaluated on 10,000 examples each from the SNLI and XSUM datasets.}
%     \label{fig:varyinggaps} 
% \end{figure*}

\begin{figure*}[]
    \centering
    \includegraphics[width=1.0\linewidth]{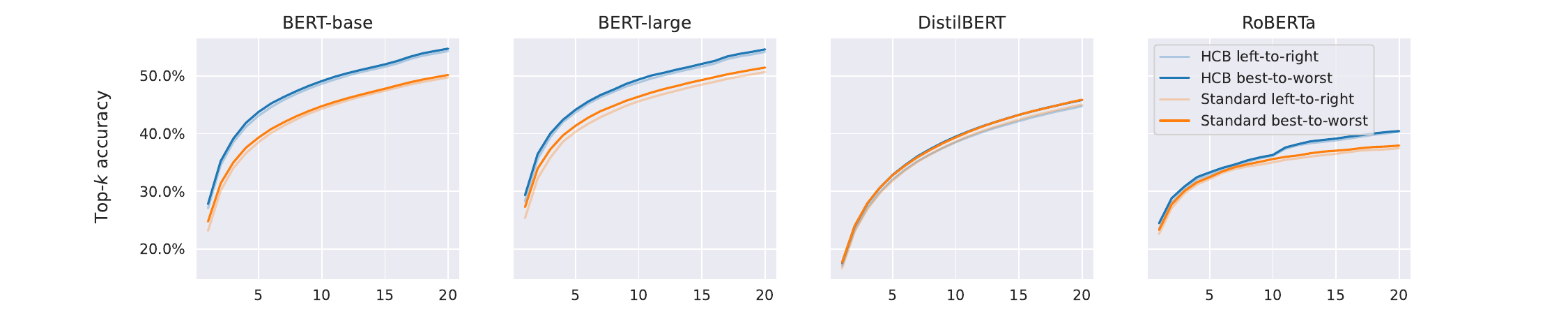}
    \caption{Performance of HCB beam search versus standard beam search methods across three English MLMs and DistilBERT, with beam size 20.}
    \label{fig:eng_models}
\end{figure*}

\begin{figure*}[h!]
    \centering
    \includegraphics[width=0.35\linewidth]{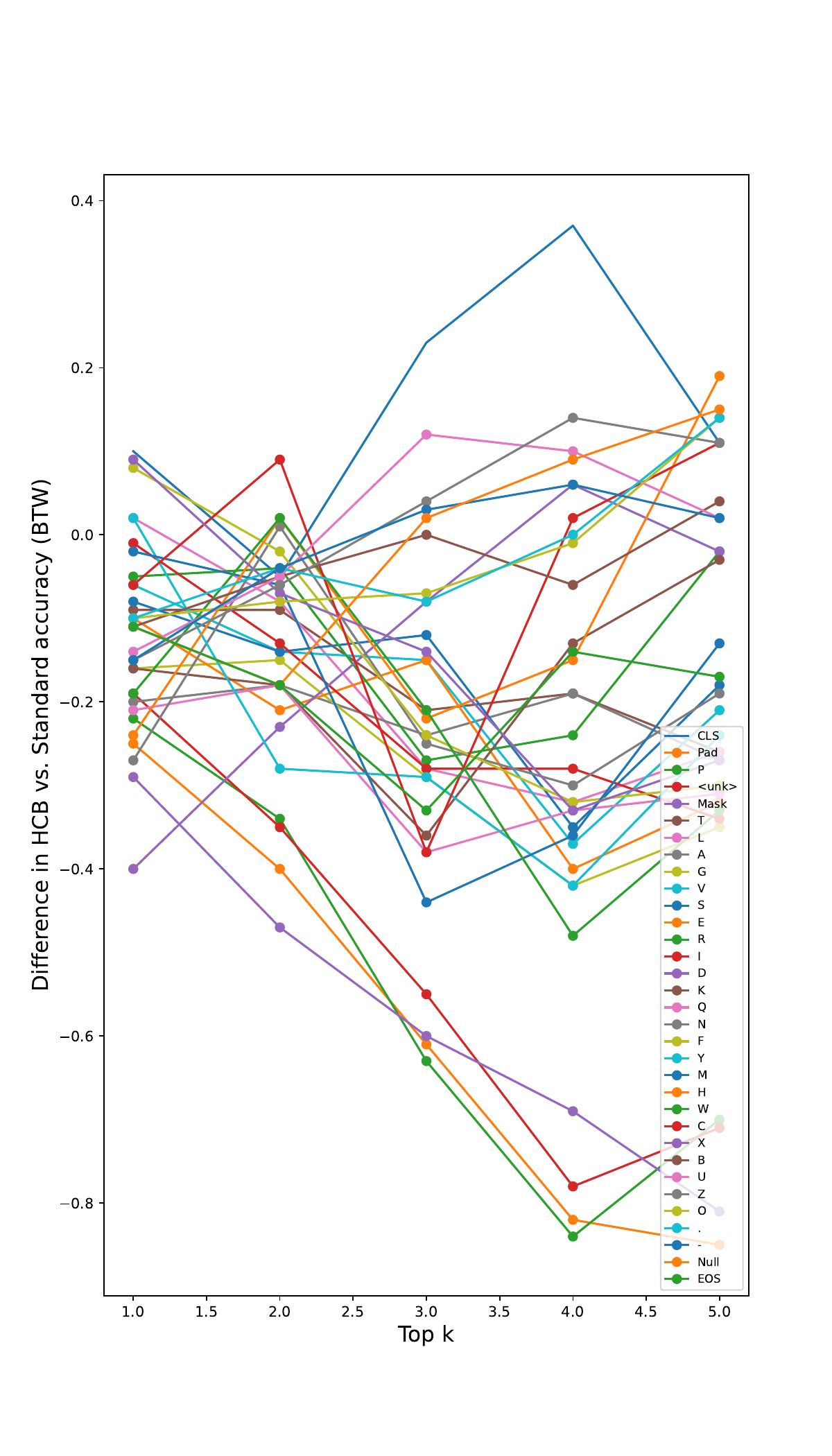}
    \caption{Performance of HCB beam search compared to standard beam search for various choices of pivot. All pivot experiments performed with beam size of 5 and gap size of 2 on the subset UniProt dataset containing 10,000 protein sequences. 50,000 trials were run for each pivot choice.}
    \label{fig:prot_pivot}
\end{figure*}

\begin{figure*}[h!]
    \centering
    \includegraphics[width=0.65\linewidth]{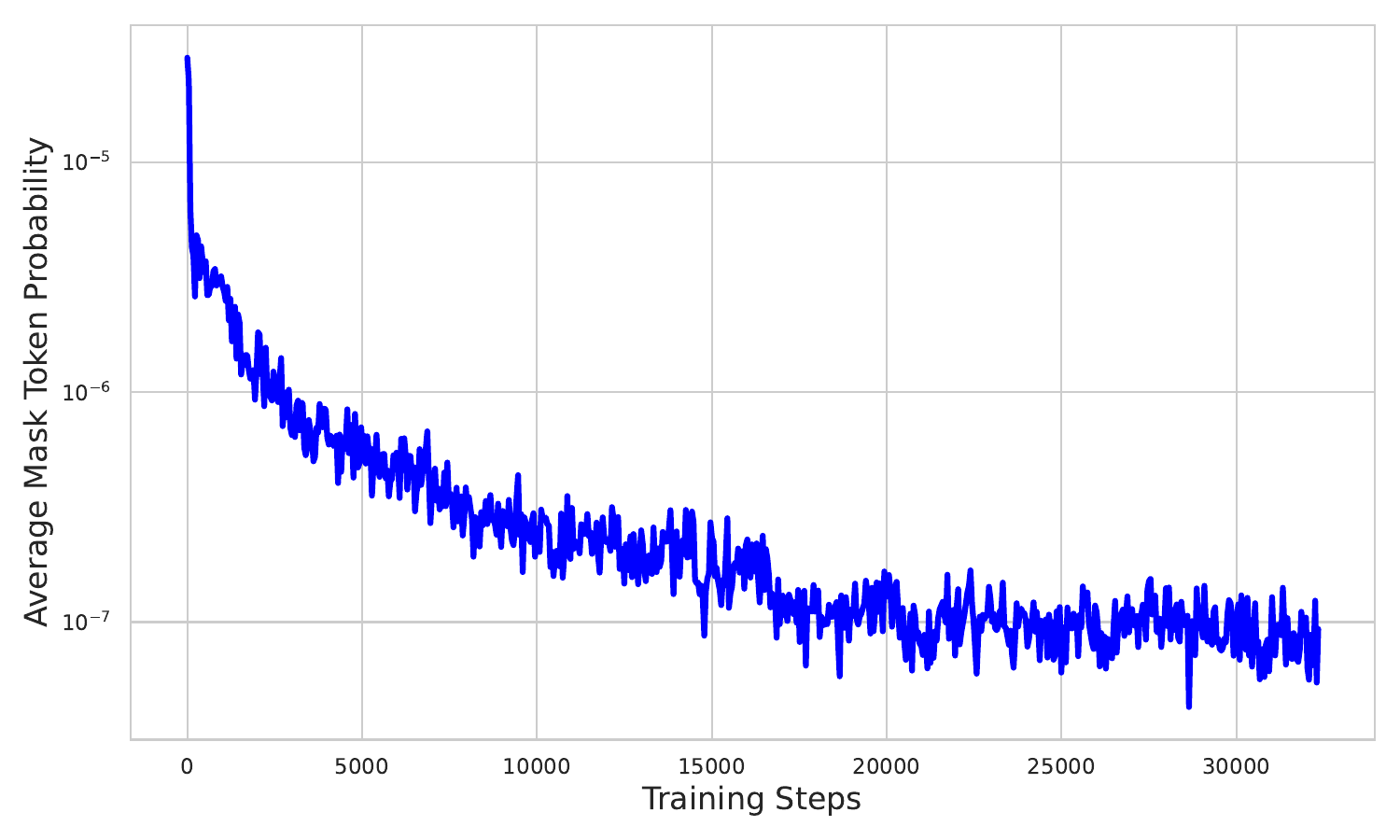}
    \caption{Average $p(\textsc{[mask]}|\cdot)$ for fixed random masks across 50 test examples over the course of BERT training.}
    \label{fig:mask_probs}
\end{figure*}

\end{document}